\def\year{2021}\relax
\documentclass[letterpaper]{article} 
\usepackage{aaai21}  
\usepackage{times}  
\usepackage{helvet} 
\usepackage{courier}  
\usepackage[hyphens]{url}  
\usepackage{graphicx} 
\usepackage{natbib}  
\usepackage{caption} 
\usepackage{amsmath}
\usepackage{amsthm}
\usepackage{amsbsy}
\usepackage{amsfonts}
\usepackage{algorithm}
\usepackage{algpseudocode}
\usepackage{multirow}
\usepackage{amssymb}
\usepackage{subfigure}
\usepackage{wrapfig}
\usepackage{float}
\usepackage{color}
\usepackage{bm}
\usepackage{pifont}
\usepackage{xcolor}

\newtheorem{proposition}{Proposition}

\newtheorem{definition}{Definition}
\newtheorem{corollary}{Corollary}

\nocopyright

\DeclareMathOperator*{\argmax}{arg\,max}

\newenvironment{proofsketch}{%
	\proof}{\endproof}

\usepackage{tikz}
\usetikzlibrary{decorations.pathreplacing,calc}
\newcommand{\tikzmark}[1]{\tikz[overlay,remember picture] \node (#1) {};}
\algnewcommand{\lIf}[1]{\State\algorithmicif\ #1\ \algorithmicthen}

\newcommand*{\AddNote}[4]{%
	\begin{tikzpicture}[overlay, remember picture]
	\draw [decoration={brace,amplitude=0.5em},decorate,ultra thick,black]
	($(#3)!(#1.north)!($(#3)-(0,1)$)$) --  
	($(#3)!(#2.south)!($(#3)-(0,1)$)$)
	node [align=center, text width=2.5cm, pos=0.5, anchor=west] {#4};
	\end{tikzpicture}
}%

\title{Mitigating Negative Side Effects via Environment Shaping}
\author{Sandhya Saisubramanian and Shlomo Zilberstein\\}
\affiliations{College of Information and Computer Sciences, \\ University of Massachusetts Amherst}

\begin{document}
	
	\maketitle

\begin{abstract}
	Agents operating in unstructured environments often produce \emph{negative side effects} (NSE), which are difficult to identify at design time. While the agent can learn to mitigate the side effects from human feedback, such feedback is often expensive and the rate of learning is sensitive to the agent's state representation.
We examine how humans can assist an agent, beyond providing feedback, and exploit their broader scope of knowledge to mitigate the impacts of NSE.
We formulate this problem as a human-agent team with decoupled objectives. The agent optimizes its assigned task, during which its actions may produce NSE. The human shapes the environment through minor reconfiguration actions so as to mitigate the impacts of the agent's side effects, without affecting the agent's ability to complete its assigned task. We present an algorithm to solve this problem and analyze its theoretical properties. Through experiments with human subjects, we assess the willingness of users to perform minor environment modifications to mitigate the impacts of NSE. Empirical evaluation of our approach shows that the proposed framework can successfully mitigate NSE, without affecting the agent's ability to complete its assigned task.
\end{abstract}

\section{Introduction}
Deployed AI agents often require complex design choices to support safe operation in the open world.
During design and initial testing, the system designer typically ensures that the agent's model includes all the necessary details relevant to its assigned task. Inherently, many other details of the environment that are unrelated to this task may be ignored. Due to this model incompleteness, the deployed agent's actions may create \emph{negative side effects} (NSE)~\cite{amodei2016concrete,saisubramanian2020avoiding}. 
For example, consider an agent that needs to push a box to a goal location as quickly as possible (Figure~\ref{fig:orig-example}). Its model includes accurate information essential to optimizing its assigned task, such as reward for pushing the box. But details such as the impact of pushing the box over a rug may not be included in the model if the issue is overlooked at system design. Consequently, the agent may push a box over the rug, dirtying it as a side effect. Mitigating such NSE is critical to improve trust in deployed AI systems.  

\begin{figure}
	\centering
	\subfigure[Initial Environment]{\includegraphics[width=1.5in]{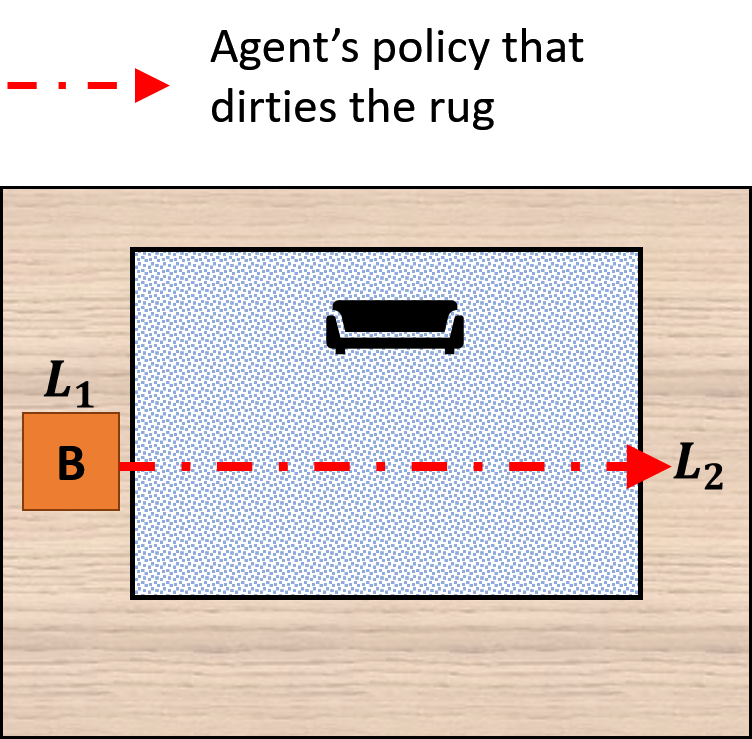}\label{fig:orig-example}} \quad
	\subfigure[Protective sheet over rug]{\includegraphics[width=1.5in]{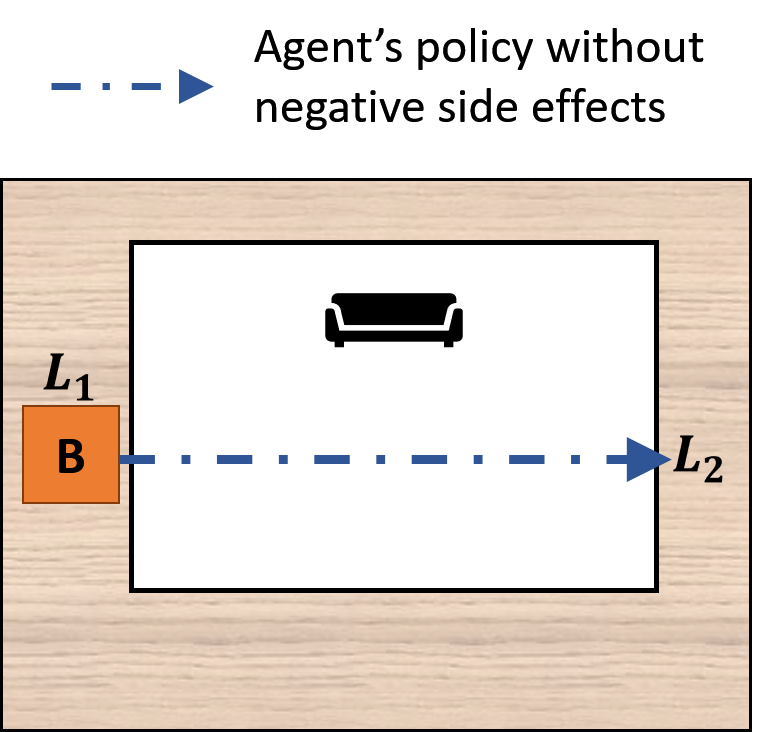}\label{fig:example-mod2}}
	\caption{Example configurations for boxpushing domain: (a) denotes the initial setting in which the actor dirties the rug when pushing the box over it; (b) denotes a modification that avoids the NSE, without affecting the actor's policy.}
	\label{fig:example}
\end{figure}

It is practically impossible to identify all such NSE during system design since agents are deployed in varied settings. Deployed agents often do not have any prior knowledge about NSE, and therefore they do not have the ability to minimize NSE.  \emph{How can we leverage human assistance and the broader scope of human knowledge to mitigate negative side effects, when agents are unaware of the side effects and the associated penalties?}

A common solution approach in the existing literature is to update the agent's model and policy by learning about NSE through feedbacks~\cite{hadfield2017inverse,zhang2018minimax,SKZijcai20}.
The knowledge about the NSE can be encoded by constraints~\cite{zhang2018minimax} or by updating the reward function~\cite{hadfield2017inverse,SKZijcai20}. These approaches have three main drawbacks. First, the agent's state representation is assumed to have all the necessary features to learn to avoid NSE~\cite{SKZijcai20,zhang2018minimax}. 
In practice, the model may include only the features related to the agent's task. This may affect the agent's learning and adaptation since the NSE could potentially be non-Markovian with respect to this limited state representation. Second, extensive model revisions will likely require suspension of operation and exhaustive evaluation before the system can be redeployed. This is inefficient when the NSE is not safety-critical, especially for complex systems that have been carefully designed and tested for critical safety aspects. Third, updating the agent's policy may not mitigate NSE that are unavoidable. For example, in the boxpushing domain, NSE are unavoidable if the entire floor is covered with a rug. 

The key insight of this paper is that agents often operate in environments that are configurable, which can be leveraged to mitigate NSE. We propose  \emph{\textbf{environment shaping}} for deployed agents: a process of applying modest modifications to the current environment to make it more agent-friendly and minimize the occurrence of NSE. Real world examples show that modest modifications can substantially improve agents' performance. For example, Amazon improved the performance of its warehouse robots by optimizing the warehouse layouts, covering skylights to reduce the glare, and repositioning the air conditioning units to blow out to the side~\cite{amazonwarehouse}. Recent studies show the need for infrastructure modifications such as traffic signs with barcodes and redesigned roads to improve the reliability of autonomous vehicles 
\cite{averyTraffic,hendrickson2014connected,michiganroads}.
Simple modifications to the environment have also accelerated agent learning~\cite{randlov2000shaping} and goal recognition~\cite{keren2014goal}. These examples show that environment shaping can improve an agent's performance with respect to its \emph{assigned task}. We target settings in which \emph{environment shaping can reduce NSE}, without significantly degrading the performance of the agent in completing its assigned task. 

Our formulation consists of an \emph{actor} and a \emph{designer}. The actor agent computes a policy that optimizes the completion of its assigned task and has no knowledge about NSE of its actions. The designer agent, typically a human, shapes the environment through minor modifications so as to mitigate the NSE of the actor's actions, without affecting the actor's ability to complete its task. The environment is described using a configuration file, such as a map, which is shared between the actor and the designer. Given an environment configuration, the actor computes and shares sample trajectories of its optimal policy with the designer. The designer measures the NSE associated with the actor's policy and modifies the environment by updating the configuration file, if necessary. We target settings in which the completion of the actor's task is prioritized over minimizing NSE. 
The sample trajectories and the environment configuration are the only knowledge shared between the actor and the designer. 

The advantages of this decoupled approach to minimize NSE are: (1) robustness---it can handle settings in which the actor is unaware of the NSE and does not have the necessary 
state representation to effectively learn about NSE; and (2) bounded-performance---by controlling the space of valid modifications, bounded-performance of the actor with respect to its assigned task can be guaranteed.  

Our primary contributions are: (1) introducing a novel framework for environment shaping to mitigate the undesirable side effects of agent actions; (2) presenting an algorithm for shaping and analyzing its properties; (3) providing the results of a human subjects study to assess the attitudes of users towards negative side effects and their willingness to shape the environment; and (4) empirical evaluation of the approach on two domains. 

\section{Actor-Designer Framework}
The problem of mitigating NSE is formulated as a collaborative actor-designer framework with decoupled objectives. The problem setting is as follows. The \emph{actor} agent operates based on a Markov decision process (MDP) $M_a$ in an environment that is configurable and described by a configuration file $E$, such as a map. The agent's model $M_a$ includes the necessary details relevant to its assigned task, which is its primary objective $o_P$. A factored state representation is assumed. The actor computes a policy $\pi$ that optimizes $o_P$. Executing $\pi$ may lead to NSE, unknown to the actor. The \emph{environment designer}, typically the user, measures the impact of NSE associated with the actor's $\pi$ and shapes the environment, if necessary. The actor and the environment designer share the configuration file of the environment, which is updated by the environment designer to reflect the modifications. 
Optimizing $o_P$ is prioritized over minimizing the NSE. Hence, shaping is performed \emph{in response} to the actor's policy. In the rest of the paper, we refer to the environment designer simply as designer---not to be confused with the designer of the agent itself.

Each modification is a sequence of design actions. An example is $\{ move(table,l_1,l_2),remove(rug)\}$, which moves the table from location $l_1$ to $l_2$ and removes the rug in the current setting. We consider the set of modifications to be finite since the environment is generally optimized for the user and the agent's primary task~\cite{shah2019preferences} and the user may not be willing to drastically modify it. Additionally, the set of modifications for an environment is included in the problem specification since it is typically controlled by the user and rooted in the NSE they want to mitigate. 

We make the following assumptions about the nature of NSE and the modifications: (1) NSE are undesirable but not safety-critical, and its occurrence does not affect the actor's ability to complete its task; (2) the start and goal conditions of the actor are fixed and cannot be altered, so that the modifications do not alter the agent's task; and (3) modifications are applied tentatively for evaluation purposes and the environment is reset if the reconfiguration affects the actor's ability to complete its task or the actor's policy in the modified setting does not minimize the NSE. 

\begin{definition}
	An actor-designer framework to mitigate negative side effects (\textbf{\emph{AD-NSE}}) is defined by $\langle E_0, \mathcal{E}, M_a, M_d,  \delta_A, \delta_D \rangle $ with:
	\begin{itemize}
		\item $E_0$ denoting the initial environment configuration;
		\item $\mathcal{E}$ denoting a finite set of possible reconfigurations of $E_0$;
		\item $M_a = \langle S,A,T,R,s_0, s_G\rangle$ is the actor's MDP with a discrete and finite state space $S$, discrete actions $A$, transition function $T$, reward function $R$, start state $s_0 \in S$, and a goal state $s_G \in S$;
		\item $M_d = \langle \Omega, \Psi,  C, N \rangle$ is the model of the designer with 
		\begin{itemize}
			\item $\Omega$ denoting a finite set of valid modifications that are available for $E_0$, including $\emptyset$ to indicate that no changes are made;
			\item $\Psi: \mathcal{E} \times \Omega \rightarrow \mathcal{E}$ determines the resulting environment configuration after applying a modification $\omega \in \Omega$ to the current configuration, and is denoted by $\Psi(E,\omega)$;
			\item $C: \mathcal{E} \times \Omega \rightarrow \mathbb{R}$ is a cost function that specifies the cost of applying a modification to an environment, denoted by $C(E,\omega)$, with $C(E,\emptyset)\!=\!0, \forall E\!\in\!\mathcal{E}$;  
			\item $N = \langle \pi , E, \zeta \rangle $ is a model specifying the penalty for negative side effects in environment $E \in \mathcal{E}$ for the actor's policy $\pi$, with $\zeta$ mapping states in $\pi$ to $E$ for severity estimation.
		\end{itemize}
		\item $\delta_A\!\ge\!0 $ is the actor's slack, denoting the maximum allowed deviation from the optimal value of $o_P$ in $E_0$, when recomputing its policy in a modified environment; and
		\item $\delta_D \ge 0$ indicates the designer's NSE tolerance threshold.
	\end{itemize}
\end{definition}

\paragraph{\textbf{Decoupled objectives}}  The actor's objective is to compute a policy $\pi$ that maximizes its expected reward for $o_P$, from the start state $s_0$ in the current environment configuration $E$:
\begin{align}
&\max_{\pi \in \Pi} V^{\pi}_P (s_0 \vert E), \nonumber \\
 V^{\pi}_P (s) = R(s,\pi(s)) &+ \sum_{s' \in S} T(s,\pi(s),s') V^{\pi}_P (s'), \forall s \in S.\nonumber
\end{align}
When the environment is modified, the actor recomputes its policy and may end up with a longer path to its goal.  The slack $\delta_A$ denotes the maximum allowed deviation from the optimal expected reward in $E_0$, 
$V^*_P (s_0 \vert E_0)$, to facilitate minimizing the NSE via shaping. A policy $\pi'$ in a modified environment $E'$ satisfies the slack $\delta_A$ if 
\[V^*_P (s_0 \vert E_0) - V^{\pi'}_P (s_0 \vert E') \leq \delta_A. \]

Given the actor's $\pi$ and environment configuration $E$, the designer first estimates its corresponding NSE and the associated penalty, denoted by $N_{\pi}^E$. The environment is modified if $N_{\pi}^{E} > \delta_D$, where $\delta_D$ is the designer's tolerance threshold for NSE, \emph{assuming $\pi$ is fixed}.  Given $\pi$ and a set of valid modifications $\Omega$, the designer selects a modification that maximizes its utility: 
\begin{align} 
&\max_{\omega \in \Omega} U_{\pi}(\omega) \nonumber \\ 
U_{\pi}(\omega) =&  \underbrace{\big(N_{\pi}^{E} - N_{\pi}^{\Psi(E,\omega)} \big)}_\text{reduction in NSE} - C(E,\omega). \label{designer-obj} \end{align}

Typically the designer is the user of the system and their preferences and tolerance for NSE is captured by the utility of a modification. The designer's objective is to balance the tradeoff between minimizing the NSE and the cost of applying the modification. It is assumed that the cost of the modification is measured in the same units as the NSE penalty. The cost of a modification may be amortized over episodes of the actor performing the task in the environment. 
The following properties are used to guarantee bounded-performance of the actor when shaping the environment to minimize the impacts of NSE. 

\begin{definition}
	Shaping is \textbf{admissible} if it results in an environment configuration in which (1) the actor can complete its assigned task, given $\delta_A$ and (2) the NSE does not increase, relative to $E_0$. \label{defn:admissible}
\end{definition}

\begin{definition}
	Shaping is \textbf{proper} if (1) it is admissible and (2) reduces the actor's NSE to be within $\delta_D$.
\label{defn:proper}
\end{definition}

\begin{definition}
	A \textbf{robust} shaping results in an environment configuration $E$ where \emph{all} valid policies of the actor, given $\delta_A$, produce NSE within $\delta_D$. That is, $N^E_{\pi} \le \delta_D$, $\forall \pi: V^*_P(s_0\vert E_0) - V^{\pi}_P(s_0\vert E) \le \delta_A $.  \label{defn:robust}
\end{definition}

\paragraph{\textbf{Shared knowledge}}
The actor and the designer do not have details about each other's model. Since the objectives are decoupled, knowledge about the exact model parameters of the other agent is not required. Instead, the two key details that are necessary to achieve collaboration are shared: configuration file describing the environment and the actor's policy. The shared configuration file, such as the map of the environment, allows the designer to effectively communicate the modifications to the actor. The actor's policy is required for the designer to shape the environment. Compact representation of the actor's policy is a practical challenge in large problems~\cite{guestrin2001max}. Therefore, instead of sharing the complete policy, the actor provides a finite number of demonstrations $\mathcal{D} = \{\tau_1, \tau_2,\ldots,\tau_n\}$ of its optimal policy for $o_P$ in the current environment configuration. Each demonstration $\tau$ is a trajectory from start to goal, following $\pi$. Using $\mathcal{D}$, the designer can extract the actor's policy by associating states with actions observed in $\mathcal{D}$, and measure its NSE. The designer does not have knowledge about the details of the agent's model but is assumed to be aware of the agent's objectives and its general capabilities of the agent. This makes it straightforward to construe the agent’s trajectories. Naturally, increasing the number and diversity of sample trajectories that cover the actor's reachable states helps the designer improve the accuracy of estimating the actor's NSE and select an appropriate modification. If $\mathcal{D}$ does not starve any reachable state, following $\pi$, then the designer can extract the actor's complete policy.

\section{Solution Approach}
Our solution approach for solving AD-NSE, described in Algorithm~\ref{algo}, proceeds in two phases: a planning phase and a shaping phase. In the planning phase (Line 4), the actor computes its policy $\pi$ for $o_P$ in the current environment configuration and generates a finite number of sample trajectories $\mathcal{D}$, following $\pi$. The planning phase ends with disclosing $\mathcal{D}$ to the designer. The shaping phase (Lines 7-15) begins with the designer associating states with actions observed in $\mathcal{D}$ to extract a policy $\hat{\pi}$, and estimating the corresponding NSE penalty, denoted by $N_{\hat{\pi}}^E$. If $N_{\hat{\pi}}^E > \delta_D$, the designer applies a utility-maximizing modification and updates the configuration file. The planning and shaping phases alternate until the NSE is within $\delta_D$ or until all possible modifications have been tested. Figure~\ref{fig:leader-follower} illustrates this approach.

\begin{figure}[t]
	\centering
	\includegraphics[scale=0.12]{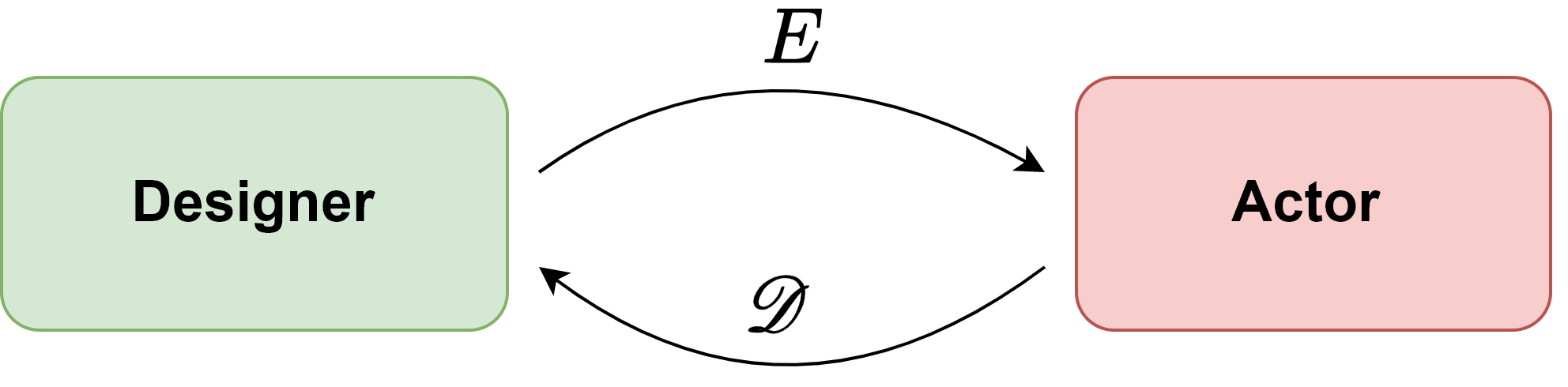}
	\caption{Actor-designer approach for mitigating NSE.}
	\label{fig:leader-follower}
\end{figure}

\begin{algorithm}[t]
	\begin{algorithmic}[1]
		\Require $\langle E_0, \mathcal{E}, M_a, M_d,  \delta_A, \delta_D \rangle $: AD-NSE
		\Require $d$: Number of sample trajectories
		\Require $b$: Budget for evaluating modifications
		\State $E^* \leftarrow E_0$ \Comment Initialize best configuration 
		\State $E \leftarrow E_0$
		\State $n \leftarrow \infty$
		
		\State $\mathcal{D} \leftarrow $ Solve $M_a$ for $E_0$ and sample $d$ trajectories 
		
		\State $\bar{\Omega} \leftarrow$ Diverse\_modifications($b,\Omega, M_d,E_0$)

		\While {$\vert \bar{\Omega} \vert >0 $}
		\State $\hat{\pi} \leftarrow $ Extract policy from $\mathcal{D}$ \tikzmark{top1} \tikzmark{right}
		\If {$N^{E}_{\hat{\pi}} < n$ }
		\State $n \leftarrow N^{E}_{\hat{\pi}}$
		\State $E^* \leftarrow E$ 
		\EndIf
		
		\lIf{$N^{E}_{\hat{\pi}} \le \delta_D $} break
		
		\State $\omega^* \leftarrow \argmax_{\omega \in \bar{\Omega}} U_{\hat{\pi}}(\omega)$  
		
		\lIf{$\omega^* = \emptyset$} break
		\State $\bar{\Omega} \leftarrow \bar{\Omega} \setminus \omega^*$ \tikzmark{bottom1}	
		\State $\mathcal{D'} \leftarrow $ Solve $M_a$ for $\Psi(E_0, \omega^*)$ and sample $d$ trajectories
		
		\If{$\mathcal{D'} \neq \{\}$}
		\State $\mathcal{D} \leftarrow \mathcal{D'}$
		\State $E \leftarrow \Psi(E_0, \omega^*)$
		\EndIf
		\EndWhile
		\State \Return $E^*$
	\end{algorithmic}
	\caption{\textbf{Environment shaping to mitigate NSE}}
	\label{algo}
	\AddNote{top1}{bottom1}{right}{Shaping phase}
\end{algorithm}

The actor returns $\mathcal{D} = \{\}$ when the modification affects its ability to reach the goal, given $\delta_A$. Modifications are applied tentatively for evaluation and the environment is reset if the actor returns $\mathcal{D} = \{\}$ or if the reconfiguration does not minimize the NSE. Therefore, all modifications are applied to $E_0$ and it suffices to test each $\omega$ without replacement as the actor always calculates the corresponding optimal policy.  When $M_a$ is solved approximately, without bounded-guarantees, it is non-trivial to verify if $\delta_A$ is violated but the designer selects a utility-maximizing $\omega$ corresponding to this policy. 
Algorithm~\ref{algo} terminates when at least one of the following conditions is satisfied: (1) NSE impact is within $\delta_D$; (2) every $\omega$ has been tested; or (3) the utility-maximizing option is no modification, $\omega^*\!=\!\emptyset$, indicating that cost exceeds the reduction in NSE or no modification can reduce the NSE further. When no modification reduces the NSE to be within $\delta_D$, configuration with the least NSE is returned.

Though the designer calculates the utility for all modifications, using $N$, there is an implicit pruning in the shaping phase since only utility-maximizing modifications are evaluated (Line 16).  However, when multiple modifications have the same cost and produce similar environment configurations, the algorithm will alternate multiple times between planning and shaping to evaluate all these modifications, which is  $\vert \Omega \vert$ in the worst case. To minimize the number of evaluations in settings with large $\Omega$, consisting of multiple similar modifications, we present a greedy approach to identify and evaluate diverse modifications.

\begin{algorithm}[t]
	\begin{algorithmic}[1]
		\State $\bar{\Omega} \leftarrow \Omega$
		\lIf{$b \ge \vert \Omega \vert$} \Return$\Omega$ 
		\ForAll{$\omega_1, \omega_2 \in \bar{\Omega}$}
		\If{similarity($\omega_1, \omega_2) \le \epsilon$}
		\State $\bar{\Omega}$ = $\bar{\Omega} \setminus \argmax_{\omega} (C(E_0, \omega_1), C(E_0, \omega_2))$
		\lIf{$\vert \bar{\Omega} \vert = b $} \Return$\bar{\Omega}$ 
		\EndIf
		\EndFor
		
	\end{algorithmic}
	\caption{\textbf{Diverse\_modifications}($b,\Omega, M_d,E_0$)}
	\label{clustering}
\end{algorithm}

\vspace{6pt}
\noindent \textbf{Selecting diverse modifications$~~~$}
Let $0 <\!b\!\le \vert \Omega \vert$ denote the maximum number of modifications the designer is willing to evaluate. When $b\!<\!\vert \Omega \vert$, it is beneficial to evaluate $b$ diverse modifications. Algorithm~\ref{clustering} presents a greedy approach to select $b$ diverse modifications. If two modifications result in similar environment configurations, the algorithm prunes the modification with a higher cost. The similarity threshold is controlled by $\epsilon$. This process is repeated until $b$ modifications are identified. Measures such as the Jaccard distance or embeddings may be used to estimate the similarity between two environment configurations.

\section{Theoretical Properties}
For the sake of clarity of the theoretical analysis, we assume that the designer shapes the environment based on the actor's exact $\pi$, without having to extract it from $\mathcal{D}$, and with $b = \vert \Omega \vert$. These assumptions allow us to examine the properties without approximation errors and sampling biases.

\begin{proposition}
	Algorithm~\ref{algo} produces \textbf{admissible} shaping. \label{prop:admissible} 
\end{proposition}
\begin{proof}
	Algorithm~\ref{algo} ensures that a modification does not negatively impact the actor's ability to complete its task, given $\delta_A$ (Lines 16-19), and stores the configuration with least NSE (Line 10). Therefore, Algorithm~\ref{algo} is guaranteed to return an $E^*$ with NSE equal to or lower than that of the initial configuration $E_0$. Hence, shaping using Algorithm~\ref{algo} is admissible.
\end{proof}

Shaping using Algorithm~\ref{algo} is not guaranteed to be proper because (1) no $\omega \in \Omega$ may be able to reduce the NSE below $\delta_D$; or (2) the cost of such a modification may exceed the corresponding reduction in NSE. 
With each reconfiguration, the actor is required to recompute its policy. We now show that there exists a class of problems for which the actor's policy is unaffected by shaping.

\begin{definition} 	
	A modification is \textbf{policy-preserving} if the actor's policy is unaltered by environment shaping.
\end{definition}
This property induces policy invariance before and after shaping and is therefore considered to be backward-compatible. Additionally, the actor is guaranteed to complete its task with $\delta_A\!=\!0$ when a modification is policy-preserving with respect to the actor's optimal policy in $E_0$. Figure~\ref{fig:dbn} illustrates the dynamic Bayesian network for this class of problems, where $r_P$ denotes the reward associated with $o_P$ and $r_N$ denotes the penalty for NSE. Let $F$ denote the set of features in the environment, with the actor's policy depending on $F_P\!\subset\!F$ and $F_D\!\subset\!F$ altered by the designer's modifications. Let $\vec{f}$  and $\vec{f'}$ be the feature values before and after environment shaping. Given the actor's  $\pi$, a policy-preserving $\omega$ follows $F_D\!= F\!\setminus F_P$, ensuring $\vec{f_P}\!= \vec{f'_P}$. 

\begin{figure}
	\centering
	\includegraphics[height=1.7in]{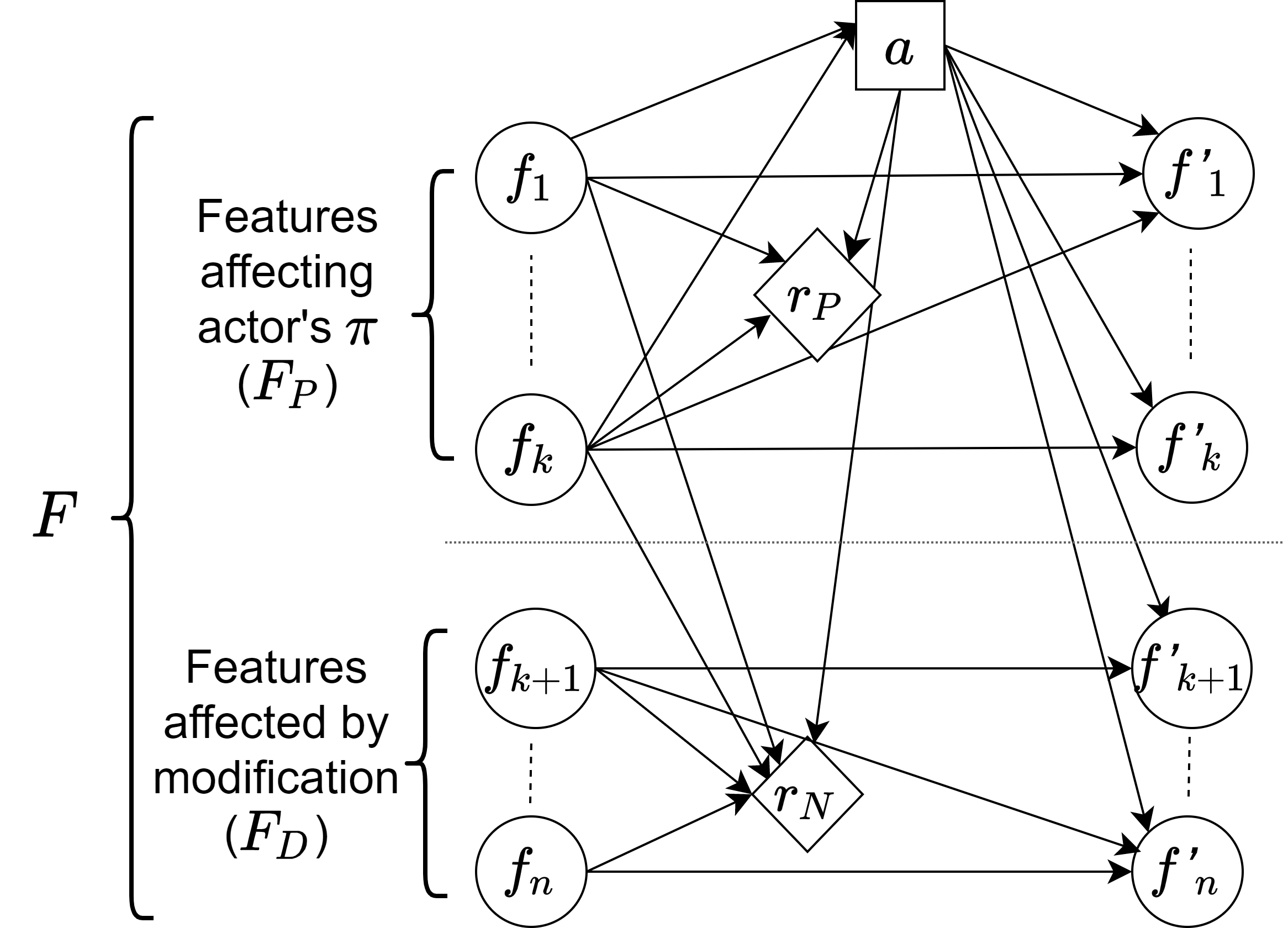}
	\caption{A dynamic Bayesian network description of a policy-preserving modification.}
	\label{fig:dbn}
\end{figure}
\begin{proposition}
	Given an environment configuration $E$ and actor's policy $\pi$, a modification $\omega \in \Omega$ is guaranteed to be \textbf{policy-preserving} if $F_D = F \setminus F_P$. \label{prop:policy-preserve}
\end{proposition}
\begin{proof}
	We prove by contradiction. Let $\omega  \in \Omega$ be a modification consistent with $F_D \cap F_P = \emptyset$ and $F = F_P \cup F_D$. Let $\pi$ and $\pi'$ denote the actor's policy before and after shaping with $\omega$ such that $\pi\!\neq \pi'$. Since the actor always computes an optimal policy for $o_P$ (assuming fixed strategy for tie-breaking), a difference in the policy indicates that at least one feature in $F_P$ is affected by $\omega$, $\vec{f_P}\!\ne\!\vec{f'_P}$. This is a contradiction since $F_D\cap F_P\!=\!\emptyset$. Therefore $\vec{f_P}\!=\!\vec{f'_P}$ and $\pi\!=\!\pi'$. Thus, $\omega\!\in\!\Omega$ is policy-preserving when $F_D = F \setminus F_P$.
\end{proof}

\begin{corollary}
	A policy-preserving modification identified by Algorithm~\ref{algo} guarantees \emph{admissible} shaping. 
\end{corollary}
Furthermore, a policy-preserving modification results in robust shaping when $Pr(F_D\! =\!\vec{f_D})\!=\!0$, $\forall \vec{f}\!=\!\vec{f_P}\!\cup\!\vec{f_D}$.
The policy-preserving property is sensitive to the actor's state representation. However, many real-world problems exhibit this feature independence. In the boxpushing example, the actor's state representation may not include details about the rug since it is not relevant to the task. Moving the rug to a different part of the room that is not on the actor's path to the goal is an example of a policy-preserving and admissible shaping. Modifications such as removing the rug or covering it with a protective sheet are policy-preserving and robust shaping since there is no direct contact between the box and the rug, for all policies of the actor~(Figure~\ref{fig:example-mod2}).

\paragraph{Relation to Game Theory}  Our solution approach with decoupled objectives can be viewed as a game between the actor and the designer. The action profile or the set of strategies for the actor is the policy space $\Pi$, with respect to $o_P$ and $\mathcal{E}$, with payoffs defined by its reward function $R$. The action profile for the designer is the set of all modifications $\Omega$ with payoffs defined by its utility function $U$. In each round of the game, the designer selects a modification that is a best response to the actor's policy $\pi$:
\begin{align}
b_D(\pi) = \{\omega \in \Omega \vert \forall \omega' \in \Omega, U_{\pi}(\omega) \geq U_{\pi}(\omega') \}. \label{eqn:br-designer}
\end{align}
The actor's best response is its optimal policy for $o_P$ in the current environment configuration, given $\delta_A$: 
\begin{align}
b_A(E) = \{\pi \in \Pi \vert \forall \pi' \in \Pi, V^{\pi}_P(s_0) \geq V^{\pi'}_P(s_0) ~\land \nonumber \\ V^{\pi_0}_P(s_0) - V^{\pi}_P(s_0) \leq \delta_A \}, \label{eqn:br-actor}
\end{align}
where $\pi_0$ denotes the optimal policy in $E_0$ before initiating environment shaping. 
These best responses are pure strategies since there exists a deterministic optimal policy for a discrete MDP and a modification is selected deterministically. Proposition~\ref{thrm:extensive} shows that AD-NSE is an extensive-form game, which is useful to understand the link between two often distinct fields of decision theory and game theory. This also opens the possibility for future work on game-theoretic approaches for mitigating NSE.

\begin{proposition}
	An AD-NSE with policy constraints induced by Equations~\ref{eqn:br-designer} and~\ref{eqn:br-actor} induces an equivalent extensive form game with incomplete information, denoted by $\mathcal{N}$. \label{thrm:extensive}
\end{proposition}
\begin{proofsketch}
	Our solution approach alternates between the planning phase and the shaping phase. This induces an extensive form game $\mathcal{N}$ with strategy profiles $\Omega$ for the designer and $\Pi$ for the actor, given start state $s_0$. However, each agent selects a best response based on the information available to it and its payoff matrix, and is unaware of the strategies and payoffs of the other agent. The designer is unaware of how its modifications may impact the actor's policy until the actor recomputes its policy and the actor is unaware of the NSE of its actions. Hence this is an extensive form game with incomplete information.
	\end{proofsketch}
Using Harsanyi transformation, $\mathcal{N}$ can be converted to a Bayesian extensive-form game, assuming the players are rational~\cite{harsanyi1967games}. This converts $\mathcal{N}$ to a game with complete but imperfect information as follows. Nature selects a type $\tau_i$, actor or designer, for each player, which determines its available actions and its payoffs parameterized by $\theta_A$ and $\theta_D$ for the actor and designer, respectively. Each player knows its type but does not know $\tau$ and  $\theta$ of the other player. If the players maintain a probability distribution over the possible types of the other player and update it as the game evolves, then the players select a best response, given their strategies and beliefs. This satisfies consistency and sequential rationality property, and therefore there exists a perfect Bayesian equilibrium for $\mathcal{N}$. 

\section{Extension to Multiple Actors}
Our approach can be extended to settings with multiple actors and one designer, with slight modifications to Equation~\ref{designer-obj} and Algorithm~\ref{algo}. This models large environments such as infrastructure modifications to cities. The actors are assumed to be homogeneous in their capabilities, tasks, and models, but may have different start and goal states. For example, consider multiple self-driving cars that have the same capabilities, task, and model but are navigating between different locations in the environment. When optimizing travel time, the vehicles may travel at a high velocity through potholes, resulting in a bumpy ride for the passenger as a side effect. Driving fast through deep potholes may also damage the car. If the road is rarely used, the utility-maximizing design may be to reduce the speed limit for that segment. If the road segment with potholes is frequently used by multiple vehicles, filling the potholes reduces the NSE for all actors. 
The designer's utility for an $\omega$, given $K$ actors, is:
\[U_{\vec{\pi}}(\omega) = \sum_{i = 1}^{K} \Big(N^E_{\pi_i} - N^{\Psi(E,\omega)}_{\pi_i} \Big) - C(E,\omega), \]

with $\vec{\pi}$ denoting policies of the $K$ actors. More complex ways of estimating the utility may be considered. The policy is recomputed for all the actors when the environment is modified. It is assumed that shaping does not introduce a multi-agent coordination problem that did not pre-exist. If a modification impacts the agent coordination for $o_P$, it will be reflected in the slack violation. Shaping for multiple actors requires preserving the slack guarantees of all actors. A policy-preserving modification induces policy invariance for all the actors in this setting.

\section{Experiments}
We present two sets of experimental results. First, we report the results of our user study conducted specifically to validate two key assumptions of our work: (1) users are willing to engage in shaping and (2) users want to mitigate NSE even when they are not safety-critical. Second, we evaluate the effectiveness of shaping in mitigating NSE on two domains in simulation. The user study complements the experiments that evaluate our algorithmic contribution to solve this problem.

\subsection{User Study}
\noindent \textbf{Setup}~~~ We conducted two IRB-approved surveys focusing on two domains \emph{similar} to the settings in our simulation experiments. The first is a household robot with capabilities such as cleaning the floor and pushing boxes. The second is an autonomous driving domain. We surveyed the participants to understand their attitudes to NSE such as the household robot spraying water on the wall when cleaning the floor, the autonomous vehicle (AV) driving fast through potholes which results in a bumpy ride for the users, and the AV slamming the brakes to halt at stop signs which results in sudden jerks for the passengers.
We recruited 500 participants on Amazon Mturk to complete a pre-survey questionnaire to assess their familiarity with AI systems and fluency in English. Based on the pre-survey responses, we invited 300 participants aged above 30 to complete the main survey, since they are less likely to game the system~\cite{downs2010your}. Participants were required to select the option that best describes their attitude towards NSE occurrence. Responses that were incomplete or with a survey completion time of less than one minute were discarded and 180 valid responses for each domain are analyzed.

\vspace{6pt}
\noindent \textbf{Tolerance to NSE$~~~$} We surveyed participants to understand their attitude towards NSE that are not safety-critical and whether such NSE affects the user's trust in the system. 

For the household robot setting, $76.66\%$ of the participants indicated willingness to tolerate the NSE. For the driving domain, $87.77\%$ were willing to tolerate milder NSE such bumpiness when the AV drives fast through potholes and $57.22\%$ were willing to tolerate relatively severe NSE such as hard braking at a stop sign. Participants were also required to enter a tolerance score on a scale of 1 to 5, with 5 being the highest. Figure~\ref{fig:tolerance-score} shows the distribution of the user tolerance score for the two domains in our human subjects study. The mean, along with the $95\%$ confidence interval for the tolerance scores are $3.06 \pm 0.197$ for the household robot domain and $3.17 \pm 0.166$ for the AV domain. For household robot domain, $66.11\%$ voted a score of $\ge 3$ and $75\%$ voted a score of $\ge 3$ for the AV domain. Furthermore, $52.22\%$ respondents of the household robot survey voted that their trust in the system's abilities is unaffected by the NSE and $34.44\%$ indicated that their trust may be affected if the NSE are not mitigated over time. Similarly for the AV driving domain, $50.55\%$ voted that their trust is unaffected by the NSE and $42.78\%$ voted that NSE may affect their trust if the NSE are not mitigated over time. 
These results suggest that (1) individual preferences and tolerance of NSE varies and depends on the severity of NSE; (2) users are generally willing to tolerate NSE that are not severe or safety-critical, but prefer to reduce them as much as possible; and (3) mitigating NSE is important to improve trust in AI systems.

\begin{figure}
	\centering
	\includegraphics[scale=0.4]{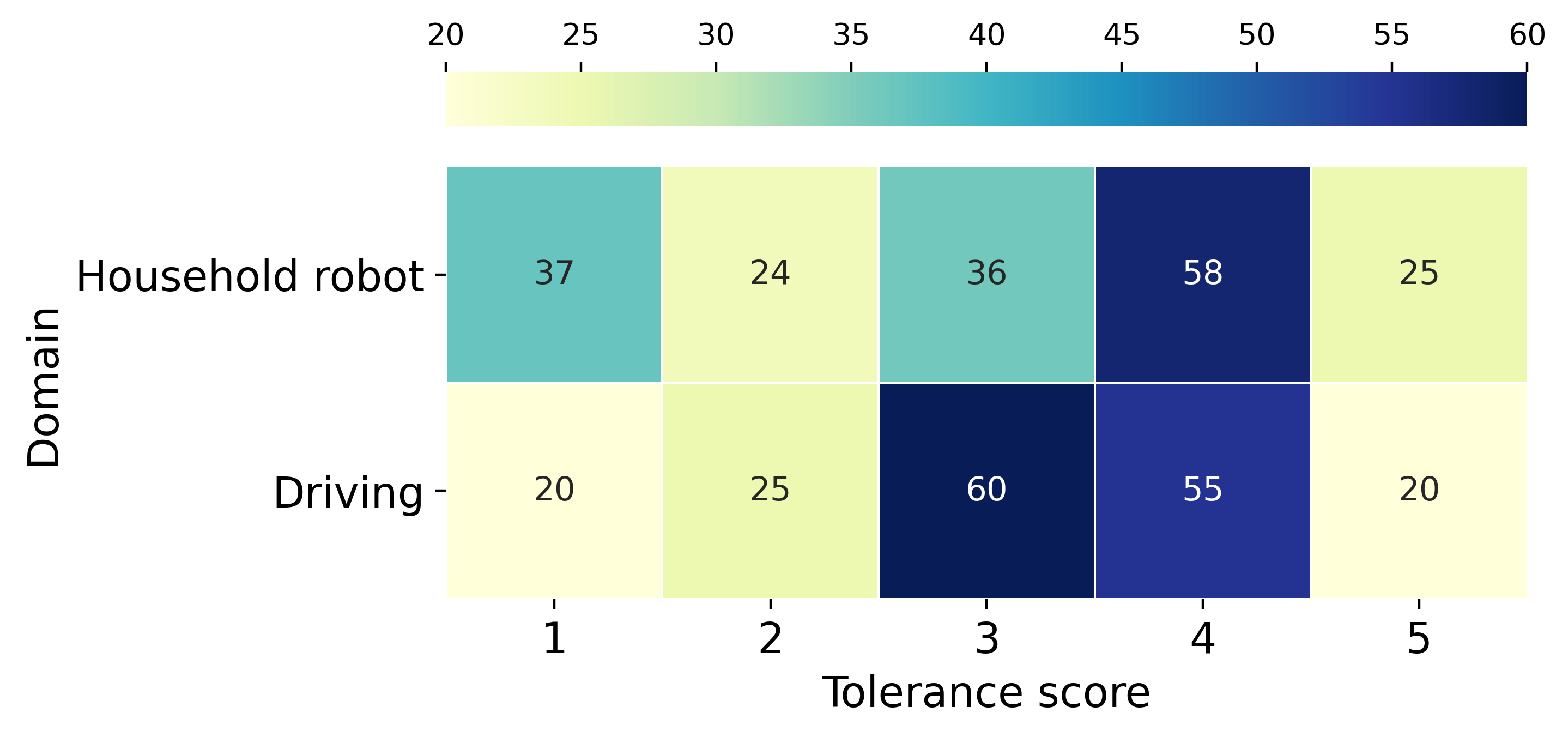}
	\caption{User tolerance score.}
	\label{fig:tolerance-score}
\end{figure}

\vspace{6pt}
\noindent \textbf{Willingness to perform shaping$~~~$} We surveyed participants to determine their willingness to perform environment shaping to mitigate the impacts of NSE. In the household robot domain, shaping involved adding a protective sheet on the surface. In the driving domain, shaping involved installing a pothole-detection sensor that detects potholes and limits the vehicle's speed. Our results show that $74.45\%$ participants are willing to engage in shaping for the household robot domain and $92.5\%$ for the AV domain. Among the respondents who were willing to install the sheet in the household robot domain, $64.39\%$ are willing to purchase the sheet ($\$10$) if it not provided by the manufacturer. Similarly, $62.04\%$ were willing to purchase the sensor for the AV, which costs $\$50$. These results indicate that (1) users are generally willing to engage in environment shaping to mitigate the impacts of NSE; and (2) many users are willing to pay for environment shaping as long as it is generally affordable, relative to the cost of the AI system.

\subsection{Evaluation in Simulation}
We evaluate the effectiveness of shaping in two domains: boxpushing (Figure~\ref{fig:example}) and AV driving (Figure~\ref{fig:driving}). The effectiveness of shaping is studied in settings with avoidable NSE, unavoidable NSE, and with multiple actors.   

\begin{figure}[t]
	\centering
	\includegraphics[height=1.25in]{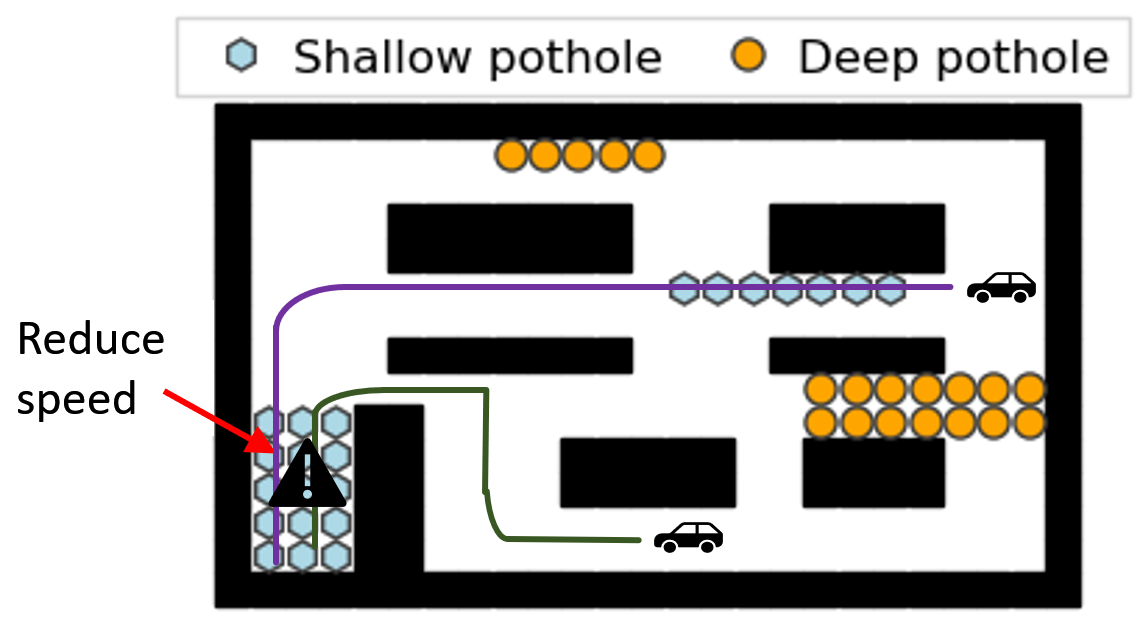}
	\caption{AV environment shaping with multiple actors.}
	\label{fig:driving}
	\vspace{-4pt}
\end{figure}

\begin{figure*}
	\centering
	\subfigure[Average NSE penalty.]{\includegraphics[scale=0.35]{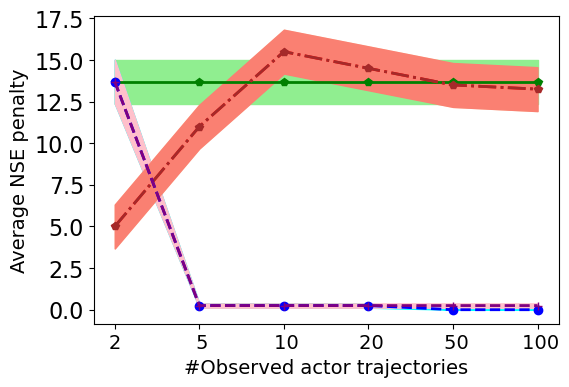} \label{fig:bp-shaping-traj}} 
	\subfigure[Expected cost of $o_P$.]{\includegraphics[scale=0.35]{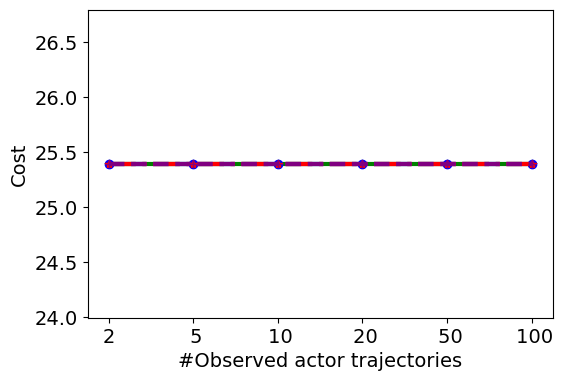} \label{fig:bp-primary}} 
	\subfigure[Average NSE penalty.]{\includegraphics[scale=0.35]{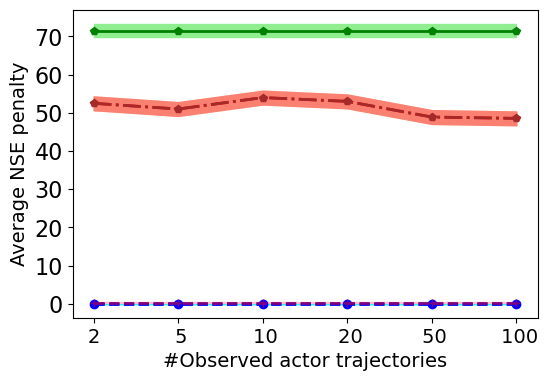} \label{fig:bp-shaping-unavoidable}} 
	\subfigure[Expected cost of $o_P$.]{\includegraphics[scale=0.35]{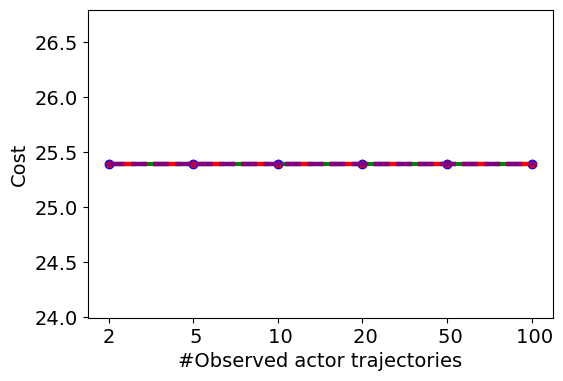} \label{fig:bp-primary-unavoidable}} 
	\subfigure[Legend for plots (a)-(d)]{\includegraphics[scale=0.42]{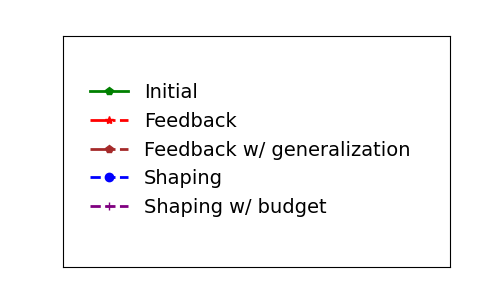}}
	\subfigure[\# Modifications evaluated when NSE are avoidable]{	\begin{tabular}[b]{|c|c|c|}
			\hline
			\# Trajectories &  Shaping & Shaping w/ budget \\ \hline
			2 & 1 & 1 \\
			5 & 2 & 1 \\
			10 & 2 & 1 \\
			20 & 2 & 1 \\
			50 & 6 & 3 \\
			100 & 6 & 3 \\
			\hline
		\end{tabular} \label{tab:modifications-bp}}
	\caption{Results on boxpushing domain with avoidable NSE (a-b) and unavoidable NSE (c-d).}
	\label{fig:bp-results}
\end{figure*}

\vspace{6pt}
\noindent \textbf{Baselines$~~~$} The performance of shaping with budget is compared with the following baselines. First is the \emph{Initial} approach in which the actor's policy is naively executed and does not involve shaping or any form of learning to mitigate NSE. Second is the \emph{shaping} with exhaustive search to select a modification, $b=\vert \Omega\vert$. Third is the \emph{Feedback} approach in which the 
gent performs a trajectory of its optimal policy for $o_P$ and the human approves or disapproves the observed actions based on the NSE occurrence~\cite{SKZijcai20}.
The agent then disables all the disapproved actions and recomputes a policy for execution. If the updated policy violates the slack, the actor ignores the feedbacks and executes its initial policy since $o_P$ is prioritized. We consider two variants of the feedback approach: with and without generalizing the gathered information to unseen situations. In \emph{Feedback w/ generalization}, the actor uses the human feedback as training data to learn a predictive model of NSE occurrence. The actor disables the actions disapproved by the human and those predicted by the model, and recomputes a policy. 

We use Jaccard distance to measure the similarity between environment configurations. A random forest classifier from the \texttt{sklearn} Python package is used for learning a predictive model. The actor's MDP is solved using value iteration. The algorithms were implemented in Python and tested on a computer with 16GB of RAM. We optimize action costs, which are negation of rewards. Values averaged over 100 trials of planning and execution, along with the standard errors, are reported for the following domains.

\vspace{6pt}
\noindent \textbf{Boxpushing$~~~$} We consider a boxpushing domain~\cite{SKZijcai20} in which the actor is required to minimize the expected time taken to push a box to the goal location (Figure~\ref{fig:example}). Each action costs one time unit. The actions succeed with probability $0.9$ and may slide right with probability $0.1$. Each state is represented as $\langle x,y,b\rangle$ where $x,y$ denote the agent's location and $b$ is a boolean variable indicating if the agent is pushing the box. Pushing the box over the rug or knocking over a vase on its way results in NSE, incurring a penalty of 5. The designers are the system users. We experiment with grid size $15\!\times 15$.

\vspace{6pt}
\noindent \textbf{Driving$~~~$}  Our second domain is based on simulated autonomous driving~\cite{SKZijcai20} in which the actor's objective is to minimize the expected cost of navigation from start to a goal location, during which it may encounter some potholes. Each state is the agent's location represented by $\langle x,y\rangle$. We consider a grid of size $25 \times 15$. The actor can move in all four directions at low and high speeds, with costs 2 and 1 respectively. Driving fast through shallow potholes results in a bumpy ride for the user, which is a mild NSE with a penalty of $2$. Driving fast through a deep pothole may damage the car in addition to the unpleasant experience for the rider and therefore it is a severe NSE with a penalty of $5$. Infrastructure management authorities are the designers for this setting and the state space is divided into four zones, similar to geographical divisions of cities for urban planning. 

\vspace{6pt}
\noindent \textbf{Available Modifications$~~~$} We consider 24 modifications for the boxpushing domain, such as adding a protective sheet over the rug, moving the vase to corners of the room, removing the rug, block access to the rug and vase area, among others. Removing the rug costs $0.4$/unit area covered by the rug, moving the vase costs $1$, and all other modifications cost $0.2$/unit. Except for blocking access to the rug and vase area, all the other modifications are policy-preserving for the actor state representation described earlier. The modifications considered for the driving domain are: reduce the speed limit in zone $i$, $1 \leq i \le 4$; reduce speed limits in all four zones; fill all potholes; fill deep potholes in all zones; reduce the speed limit in zones with shallow potholes; and fill deep potholes. Reducing the speed costs $0.5$ units per pothole in that zone and filling each pothole costs $2$ units. Reducing the speed limit disables the `move fast' action for the actor. Filling the potholes is a policy-preserving modification.

\vspace{6pt}
\noindent \textbf{Effectiveness of shaping$~~~$} The effectiveness of shaping is evaluated in terms of the average NSE penalty incurred and the expected value of $o_P$ after shaping. Figure~\ref{fig:bp-results} plots the results with $\delta_A\!=\!0$, $\delta_D\!=\!0$, and $b\!=\!3$, as the number of observed actor trajectories is increased. The results are plotted for the boxpushing domain with one actor in settings with avoidable and unavoidable NSE. Feedback budget is 500. Feedback without generalization does not minimize NSE and performs similar to \emph{initial} baseline, because the actor has no knowledge about the undesirability of actions not included in the demonstration. As a result, when an action is disapproved, the actor may execute an alternate action that is equally bad or worse than the initial action in that state, in terms of NSE. Generalizing the feedback overcomes this drawback and mitigates the NSE relatively. With at most five trajectories, the designer is able to select a \emph{policy-preserving} modification that avoids NSE. Shaping w/ budget $b=3$ performs similar to shaping by evaluating all modifications. The trend in the relative performances of the different techniques is similar for both avoidable and unavoidable NSE. Table~\ref{tab:modifications-bp} shows that shaping w/ budget reduces the number of shaping evaluations by $50\%$.

\begin{figure}
	\centering
	\includegraphics[scale=0.4]{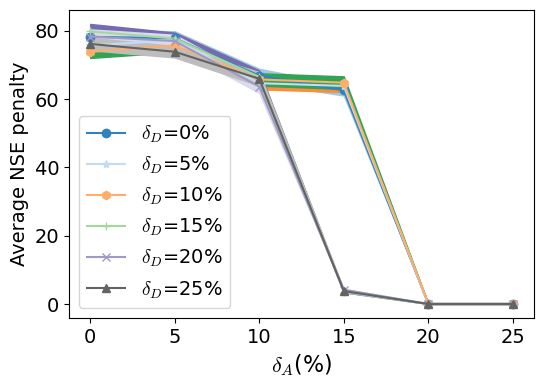}
	\vspace{-4pt}
	\caption{Effect of $\delta_A$ and $\delta_D$ on the average NSE penalty.}
	\label{fig:slack}
	\vspace{-6pt}
\end{figure}

\vspace{6pt}
\noindent \textbf{Effect of slack on shaping$~~~$} We vary $\delta_A$ and $\delta_D$, and plot the resulting NSE penalty in Figure~\ref{fig:slack}. We report results for the driving domain with a single actor and designer, and shaping based on 100 observed trajectories of the actor. We vary $\delta_A$ between 0-$25\%$ of $V_P^*(s_0|E_0)$ and $\delta_D$ between 0-$25\%$ of the NSE penalty of the actor's policy in $E_0$. Figure~\ref{fig:slack} shows that increasing the slack helps reduce the NSE, as expected. In particular, when $\delta_D\!\ge\!15\%$, the NSE penalty is considerably reduced with $\delta_A\!=\!15\%$. Overall, increasing $\delta_A$ is most effective in reducing the NSE. We also tested the effect of slack on the cost for $o_P$. 
The results showed that the cost was predominantly affected by $\delta_A$. We observed similar performance for all values of $\delta_D$ for a fixed $\delta_A$ and therefore do not include that plot. This is an expected behavior since $o_P$ is prioritized and shaping is performed in response to $\pi$.

\vspace{6pt}
\noindent \textbf{Shaping with multiple actors$~~~$} We measure the cumulative NSE penalty incurred by varying the number of actors in the environment from 10 to 300. Figure~\ref{fig:multipleactors} shows the results for the driving domain, with $\delta_D\!=\!0$ and $\delta_A\!=\!25\%$ of the optimal cost for $o_P$, for each actor. The start and goal locations were randomly generated for the actors. Shaping and feedbacks are based on 100 observed trajectories of each actor. Shaping w/ budget results are plotted with $b\!=\!4$, which is 50\%$\vert \Omega\vert$. The feedback budget for \emph{each} actor is 500. Although the feedback approach is sometimes comparable to shaping when there are fewer actors, it requires overseeing each actor and providing individual feedback, which is impractical. Overall, as we increase the number of agents, substantial reduction in NSE is observed with shaping. The time taken for shaping and shaping w/ budget are comparable up to 100 actors, beyond which we see considerable time savings of using shaping w/ budget. For 300 actors, the average time taken for shaping is 752.27 seconds and the time taken for shaping w/budget is 490.953 seconds.

\begin{figure}
	\centering
	\includegraphics[scale=0.45]{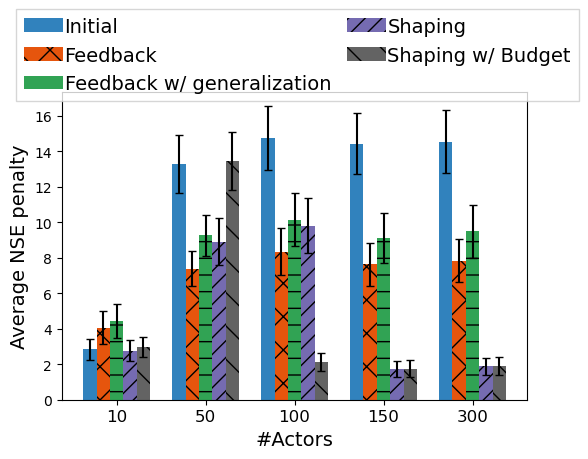}
	\vspace{-4pt}
	\caption{Results with multiple actors.}
	\label{fig:multipleactors}
	\vspace{-4pt}
\end{figure}

\section{Discussion and Future Work}
We present an actor-designer framework to mitigate the impacts of NSE by shaping the environment. The general idea of environment modification to influence the behaviors of the acting agents has been previously explored in other contexts~\cite{zhang2009general}, such as to accelerate agent learning~\cite{randlov2000shaping}, to quickly infer the goals of the actor~\cite{keren2014goal}, and to maximize the agent's reward~\cite{keren2017equi}. We study how environment shaping can mitigate NSE. We also identify the conditions under which the actor's policy is unaffected by shaping and show that our approach guarantees bounded-performance of the actor. 

Our approach is well-suited for settings where the agent is well-trained to perform its assigned task and must perform it repeatedly, but NSE are discovered after deployment. Altering the agent's model after deployment requires the environment designer to be (or have the knowledge of) the designer of the AI system. Otherwise, the system will have to be suspended and returned to the manufacturer. Our framework provides a tool for the users to mitigate NSE without making any assumptions about the agent’s model, its learning capabilities, the representation format for the agent to effectively use new information, or the designer's knowledge about the agent's model. We argue that is many contexts, it is much easier for users to reconfigure the environment than to accurately update the agent’s model. Our algorithm guides the users in the shaping process.

In addition to the greedy approach described in Algorithm~\ref{clustering}, we also tested a clustering-based approach to identify diverse modifications. In our experiments, the clustering approach performed comparably to the greedy approach in identifying diverse modifications but had a longer run time. It is likely that benefits of the clustering approach would be more evident in settings with a much larger $\Omega$, since it can quickly group similar modifications. In the future, we aim to experiment with a large set of modifications, say $\vert \Omega \vert > 100$, and compare the performance of the clustering approach with that of Algorithm~\ref{clustering}. 

When the actor's model $M_a$ has a large state space, existing algorithms for solving large MDPs may be leveraged. When $M_a$ is solved approximately, without bounded guarantees, slack guarantees cannot be established. Extending our approach to multiple actors with different models is an interesting future direction.

We conducted experiments with human subjects to validate their willingness to perform environment shaping. In our user study, we focused on scenarios that are more accessible to participants from diverse backgrounds so that we can obtain reliable responses regarding their overall attitude to NSE and shaping. In the future, we aim to evaluate whether Algorithm~\ref{algo} aligns with human judgment in selecting the best modification. Additionally, we will examine ways to automatically identify useful modifications in a large space of valid modifications.

\paragraph{Acknowledgments}
Support for this work was provided in part by the Semiconductor Research Corporation under grant \#2906.001.
\bibliography{References}

\end{document}